\documentclass{article}
\usepackage{amssymb,amsmath}
\usepackage{graphicx,enumerate}

\title{Compatible extensions and consistent closures: a fuzzy approach}

\author{Irina Georgescu \\ \footnotesize Academy of Economic Studies\\ \footnotesize Department of Economic Informatics and Cybernetics\\ \footnotesize Pia$\c{t}$a Romana No 6  R 70167, Oficiul Postal 22, Bucharest, Romania
 \footnotesize Email: irina.georgescu@csie.ase.ro}

\date{}

\begin{document}
\maketitle

\begin{abstract}
In this paper $\ast$--compatible extensions of fuzzy relations are studied, generalizing some results obtained by Duggan in case of crisp relations. From this general result are obtained as particular cases fuzzy versions of some important extension theorems for crisp relations (Szpilrajn, Hansson, Suzumura). Two
 notions of consistent closure of a fuzzy relation are introduced.

\end{abstract}


\newtheorem{definitie}{Definition}[section]
\newtheorem{propozitie}[definitie]{Proposition}
\newtheorem{remarca}[definitie]{Remark}
\newtheorem{exemplu}[definitie]{Example}
\newtheorem{intrebare}[definitie]{Open question}
\newtheorem{lema}[definitie]{Lemma}
\newtheorem{teorema}[definitie]{Theorem}
\newtheorem{corolar}[definitie]{Corollary}

\newenvironment{proof}{\noindent\textbf{Proof.}}{\hfill\rule{2mm}{2mm}\vspace*{5mm}}

\section{Introduction}

In his seminal paper \cite{szpilrajn} Szpilrajn proved an extension theorem of a strict partial order to a total strict order. This result led afterwards to other extension theorems \cite{bossert1}, \cite{bossert2}, \cite{dushnik}, \cite{suzumura1}. Szpilrajn's theorem and other extension theorems derived from economics
have been intensely used in the classical choice theory \cite{richter}, \cite{suzumura2}, \cite{weymark}.

As known, rationality is one of the main attributes which are taken into account in the analysis of economic decisions. An agent chooses between different alternatives; the rationality of the act of choice is defined by a preference relation acting on the alternatives. According to the context, the preference relation must satisfy some properties. Among these, transitivity acts extremely prolifically in classical consumer theory \cite{richter}, \cite{suzumura1}, \cite{suzumura2}. Still there exist situations when transitivity is a too strong condition. The notion of consistency was introduced by Suzumura in \cite{suzumura1} in order to weaken transitivity.

In choice theory two types of extension relations appear. The first one represents merely the inclusion: "relation $Q$ extends relation $R$" means that $R \subseteq Q$. If we interpret a relation as expressing the preference between alternatives, this type of extension preserves the preference. The second type of extension is determined by two inclusions: "relation $Q$ extends relation $R$" if $R \subseteq Q$ and $P_R \subseteq P_Q$ (here, $P_R$, $P_Q$ are the asymmetric parts of $R$, $Q$). Since the asymmetric part $P_R$ is interpreted as the strict preference, it follows that these extensions preserve the preference and the strict preference. For the second type extensions, Duggan \cite{duggan1} uses the notion "compatible extensions". Suzumura's theorem \cite{suzumura1} links the consistency with the compatible extensions: a relation is consistent iff it admits a total and transitive compatible extension.

In \cite{duggan1} Duggan establishes a  very general result which encompasses the theorems of Szpilrajn \cite{szpilrajn}, Hansson \cite{hansson}, Suzumura \cite{suzumura1}, Dushnik--Miller \cite{dushnik} and Donaldson--Weymark \cite{donaldson}  and other new particular cases.
Duggan's theorem provides total compatible extensions for various classes of crisp relations.

Paper \cite{georgescu2} has been an attempt to obtain for fuzzy relations a result similar to Duggan's extension theorem. In the setting of fuzzy sets theory associated with the G\"{o}del t-norm \cite{hajek} a concept of compatible extension of a fuzzy relation and one of consistent fuzzy relation (transitive--consistent in the terminology of \cite{bossert2}, \cite{georgescu2}) have been defined. The main result of \cite{georgescu2} is an extension theorem for fuzzy relations which generalizes  Proposition 9 from \cite{duggan1}. As particular cases fuzzy versions of Szpilrajn, Hansson and Suzumura theorems have been obtained.

In defining the meaning of compatible extension of a fuzzy relation $R$ the asymmetric part of $R$, $P_R$ comes in, notion in whose expression appears the negation corresponding to the G\"{o}del t-norm. Since this negation takes only the values $0$ and $1$, an important part of the information on the fuzzy preference relation $R$ (therefore on the act of choice) is lost.

This paper aims to define a concept of compatible extension which should avoid the negation.
One starts from the remark that, in the classical case, the definition of the compatible extension and of the consistency can be done in terms of
the Boolean implication (see Lemma 4.2). This fact suggests the definition of the $\ast$--compatible extension,
concept which is expressed using only the residuum $\rightarrow$ associated with a left--continuous t--norm $\ast$ (see \cite{hajek}, \cite{klement}). Accordingly, we shall also obtain a notion of $\ast$--consistency.

Once done this step, this idea is to redo in a fuzzy context the proof of Duggan's theorem, by essentially using the structure of residuated lattice of the interval $[0, 1]$.

In Section $2$ the main extension theorems are recalled: Szpilrajn, Hansson, Suzumura, Dushnik-Miller and Donaldson--Weymark and also
Duggan's General Extension Theorem. It is separately stated (Theorem 2.4) that part of Duggan's results
which constitutes the generalization of the theorems of Szpilrajn, Hansson and Suzumura. Section $3$ presents some definitions and properties of left--continuous t--norms, the associated residuation structure and the fuzzy relations.

Section $4$ is the core of the paper. For a left--continuous t--norm $\ast$  we define the concepts of $\ast$--compatible extension and $\ast$--consistent fuzzy relation. If $\ast$ is the G\"{o}del t-norm $\wedge$, then any $\wedge$--compatible extension is compatible (in the sense of \cite{georgescu2}) but the converse does not hold. The main results are two extension theorems. Both are fuzzy extensions of Theorem 2.4.
The first one is stated in the framework offered by a left--continuous t--norm. One proved that if the $\ast$--compatibility of the fuzzy relations is "transitive" then the fuzzy version of Theorem 2.4 is valid. For the G\"{o}del t-norm, this property is verified. One obtains then the second extension theorem, which is exactly the fuzzy form of Theorem 2.4 for the G\"{o}del t-norm.

Section $5$ concerns some particular cases of extension theorems. In Section $6$ to each fuzzy relation $R$ on $X$ one associates two fuzzy relations $R^\Delta$ and $R^\nabla$ \footnote{$R^\nabla$ has been defined in
\cite{chaudhari}.} . $R^\Delta$ and $R^\nabla$ are in general different, but both extend the notion of consistent closure of a crisp relation \cite{bossert1}. For the G\"{o}del t-norm we have $R^\Delta=R^\nabla$ and this is the smallest $\wedge$--consistent fuzzy relation on $X$ including $R$.

\section{Classical extension theorems}

In this section the extension theorems of Szpilrajn, Hansson, Suzumura, Dushnik-Miller and Donaldson--Weymark are recalled. The extension theorem of Duggan is a very general result which encompasses all these theorems.

Consider a universe $X$ of alternatives and a (crisp) binary relation $R$ on $X$. The binary relation $P_R=\{(x, y)|(x, y)\in R$ and $(y, x) \notin R\}$ is called the {\emph{asymmetric part}} of $R$. We denote by $T(R)$ the {\emph{transitive closure}} of $R$ (the smallest transitive relation containing $R$). A {\emph{preorder}} (resp. {\emph{weak order}}) on $X$ is a reflexive and transitive (resp. reflexive, transitive and total) relation on $X$. If a relation $R$ on $X$ is irreflexive and transitive then  it is called a {\emph{strict partial order}} on $X$; a strict partial order is called {\emph{strict total order}} if it is total.

Let $R, Q$ be two relations on $X$. $Q$ is called an {\emph{extension}} (resp. a {\emph{compatible extension}}) of $R$ if $R \subseteq Q$ (resp. $R \subseteq Q$ and $P_R \subseteq P_Q$). $R$ is a {\emph{consistent}} relation if $T(R)$ is a compatible extension of $R$, i.e. $P_R \subseteq P_{T(R)}$. Of course a transitive relation is consistent but the converse assertion is not true.

\begin{teorema}\cite{szpilrajn}
(Szpilrajn) Any strict partial order on $X$ can be extended to a strict total order on $X$.
\end{teorema}

\begin{teorema}\cite{hansson}
(Hansson) Any preorder on $X$ can be extended to a total preorder on $X$.
\end{teorema}

\begin{teorema}\cite{suzumura1}
(Suzumura) A relation on $X$ has a total, transitive and compatible extension iff it is consistent.
\end{teorema}

For any relation $R$ and $x, y \in X$ denote $R[x, y]=R \cup \{(x, y)\}$. A family
$\mathcal{R}$ of relations on $X$ is {\emph{closed upward}} if $\bigcup \mathcal{C} \in \mathcal{R}$ for any chain $\mathcal{C} \subseteq \mathcal{R}$;  $\mathcal{R}$ is said to be {\emph{arc--receptive}} if for all transitive relations $R \in \mathcal{R}$ and for all distinct $x \not= y$, $(y, x) \notin R$ implies $T(R[x, y]) \in \mathcal{R}$.

We recall two results of \cite{duggan1}.

\begin{teorema}\cite{duggan1}
(Duggan) Let $\mathcal{R}$ be a closed upward and arc--receptive class of relations on $X$. If $R$ is a transitive relation in $\mathcal{R}$ then there exists a total and transitive compatible extension of $R$ in $\mathcal{R}$.
\end{teorema}

\begin{teorema}\cite{duggan1}
(Duggan's General Extension Theorem) Let $\mathcal{R}$ be a closed upward and arc--receptive class of relations on $X$. If $R$ is a consistent relation on $X$ and $T(R) \in \mathcal{R}$ then

$T(R)=\bigcap \{Q \in \mathcal{R}|Q$ is a total, transitive and compatible extension of $\mathcal{R} \}$.
\end{teorema}

\begin{remarca}
Theorem 2.4 is a part of the General Extension Theorem; it generalizes the theorems of Szpilrajn, Hansson and Suzumura.
\end{remarca}

The following two results are also particular cases of the General Extension Theorem.

\begin{teorema}\cite{dushnik}
(Dushnik--Miller) Any strict partial order on $X$ is the intersection of strict total orders in which it is embedded.
\end{teorema}

\begin{teorema}\cite{donaldson}
(Donaldson--Weymark) Any preorder on $X$ is the intersection of total preorders in which it is embedded.
\end{teorema}

\section{Preliminaries on t--norms and fuzzy relations}

In this section some definitions and basic properties on left--continuous t--norms and fuzzy relations are defined. They will be used in the next sections both to define
the formal context of the paper, and as a technical tool to prove the results.

Let $[0, 1]$ be the unit interval. For any family $\{a_i\}_{i \in I}$ of elements in $[0, 1]$, let us denote
$ \displaystyle \bigvee_{i \in I} {a_i}=\sup
\{a_i\mid i \in I\}$ and $\displaystyle \bigwedge_{i \in I}
{a_i}=\inf \{a_i\mid i \in I\}$.

In particular, for all $a, b \in [0, 1]$, $a \vee b=\sup(a, b)$ and $a \wedge b=\inf(a, b)$.

Let $\ast: [0, 1] \times [0, 1] \rightarrow [0, 1]$ be a left--continuous t-norm \cite{klement} and $\rightarrow$ its residuum:

\[
a \rightarrow b=\bigvee \{c \in [0, 1]\mid a \ast c \leq b\}.
\]

The corresponding negation $\neg$ will be defined by:

$\neg a=a \rightarrow 0=\bigvee \{c\in [0, 1]|a \ast c=0\}$.

\begin{lema}
\cite{hajek}, \cite{klement}, \cite{turunen} For any $a, b, c \in [0, 1]$ the following properties hold:
\begin{enumerate}[(1)]
\item  $ a \ast b \leq c$ iff $a \leq b \rightarrow c$;
\item  $a \ast (a \rightarrow b) \leq a \wedge b$;
\item  $a \ast b \leq a$; $a \ast b \leq b$;
\item  $b \leq a \rightarrow b$;
\item $a \leq b$ iff $a \rightarrow b=1$;
\item $a=1 \rightarrow a$;
\item $1=a \rightarrow a$;
\end{enumerate}
\end{lema}

\begin{lema}
\cite{hajek}, \cite{klement}, \cite{turunen} For any $\{a_i\}_{i \in I}\subseteq [0, 1]$, $\{b_i\}_{i \in I}\subseteq [0, 1]$ and $a \in [0, 1]$ the following properties hold:
\begin{enumerate}[(1)]
\item $\displaystyle (\bigvee_{i \in I} a_i)\ast a=\displaystyle (\bigvee_{i \in I} a_i \ast a)$;
\item $\displaystyle (\bigvee_{i \in I} a_i)\rightarrow  a= \bigwedge_{i \in I} (a_i \rightarrow a)$;
\item $\displaystyle (\bigvee_{i \in I} a_i) \ast \bigvee_{j \in J} b_j)=(\bigvee_{(i, j) \in I \times J} a_i \ast b_j)$.
\end{enumerate}
\end{lema}

\begin{lema}
\cite{hajek}, \cite{klement}, \cite{turunen} For any $a, b, c \in [0, 1]$ the following properties hold:
\begin{enumerate}[(1)]
\item $a \leq \neg b$ iff $a \ast b=0$;
\item $a \ast \neg a=0$.
\end{enumerate}
\end{lema}

For the G\"{o}del t-norm $a \ast_G b =a \wedge b=\min(a, b)$, the residuum and negation will have the form:

$a \rightarrow_G b =\left\{\begin{array}{rcl}
1& \text{if } a \leq b \\
b& \text{if } a >b
\end{array}\right.$,
$\neg a=\left\{\begin{array}{rcl}
1 & \mbox{if}& a=0 \\
0 & \mbox{if}& a >0
\end{array}\right.$

For the Lukasiewicz t--norm, $a \ast_L b=\max(0, a+b-1)$, the residuum and negation will have the form:

$a \rightarrow_L b =\min(1, 1-a+b)$, $\neg a=1-a$.

For the product t--norm, $a \ast_P =ab$, the residuum and negation will have the form:

$a \rightarrow_P b=\left\{\begin{array}{rcl}
1& \text{if } a \leq b \\
\frac{b}{a}& \text{if } a >b
\end{array}\right.$,
$\neg a=\left\{\begin{array}{rcl}
1 & \mbox{if}& a=0 \\
0 & \mbox{if}& a >0
\end{array}\right.$

\begin{remarca}
The fact that the negation of the G\"{o}del t-norm takes only the values $0$ and $1$ shows that the notions and the results in which this negation appears are very close
to the crisp mathematics and insufficiently expressive for the fuzzy modelling. Therefore in the construction of the models of fuzzy uncertainty situations it is indicated
to avoid the use of the G\"{o}del t-norm.
\end{remarca}

Now we fix a left--continuous t--norm $\ast$. A {\emph{fuzzy relation}} on a universe $X$ of alternatives is a fuzzy subset of $X^2$, i.e. a function $R: X^2 \rightarrow [0,
1]$. $R$ will be interpreted as a fuzzy preference relation: for any alternatives $x, y$, $R(x, y)$
is the degree to which it is true that $x$ is at least as good as $y$. If $R, Q$ are two fuzzy relations on $X$ then $Q$ is an {\emph{extension}} of $R$ if $R \subseteq Q$. The relation $R$ on $X$ is said to be

$\bullet$ reflexive if $R(x, x)=1$ for any $x \in X$;

$\bullet$ irreflexive if $R(x, x)=0$ for any $x \in X$;

$\bullet$ $\ast$--transitive if $R(x, y)\ast R(y, z) \leq R(x, z)$ for all $x, y, z \in X$;

$\bullet$ total if $R(x, y)>0$ or $R(y, x)>0$ for all distinct $x, y \in X$;

$\bullet$ strongly total if $R(x, y)=1$ or $R(y, x)=1$ for all distinct $x, y \in X$.

The $\ast$--{\emph{transitive closure}} $T(R)$ of $R$ is the intersection of all $\ast$--transitive extensions of $R$.

\begin{remarca}
From the above definitions one can see that the $\ast$--transitivity and the notion of $\ast$--transitive closure depend on the t--norm $\ast$.
\end{remarca}

The following lemma is well-known.

\begin{lema}
If $R$ is a fuzzy relation on $X$ then for all $x, y \in X$:

$T(R)(x, y)=R(x, y) \vee \displaystyle \bigvee_{n=1}^\infty \bigvee_{t_1, \ldots, t_n \in X} R(x, t_1) \ast \ldots \ast R(t_n, y)$.
\end{lema}

\section{Compatible extensions revisited}

Compatible extensions of fuzzy relations have been introduced in \cite{georgescu2} in the context offered by the G\"{o}del t-norm.
In this section we will propose a more appropriate concept of compatible extension. The notion of compatible extension will be given in the context of fuzzy sets theory associated with a left--continuous t--norm $\ast$ and will use the residuum instead of negation. For the case of the G\"{o}del t-norm one obtains a stronger notion than in \cite{georgescu2}. Two extension theorems which generalize Duggan's result (Theorem 2.4) are proved.
The first of them is formulated for a left--continuous arbitrary t--norm and the second one for the G\"{o}del t-norm. Their proof will use a series of preliminary lemmas and propositions on the $\ast$--compatible extensions which have an intrinsic interest too.

Let $(B, \vee, \wedge, \neg, 0, 1)$ be a Boolean algebra. The Boolean implication $\rightarrow$ in $B$ is defined by $a \rightarrow b=\neg a \vee b$. An easy computation shows that $\neg(a \wedge \neg b)=a \rightarrow b$ for all $a, b \in B$. We observe that this equality does not hold in the residuated lattice $([0, 1], \vee, \wedge, \rightarrow, 0, 1)$.

Let $X$ be a non--empty set. By identifying the (crisp) relations with their characteristic functions, any relation $R$ on $X$ can be written $R: X^2 \rightarrow L_2$, where $L_2$ is the Boolean algebra $\{0, 1\}$. Then the asymmetric part $P_R : X^2 \rightarrow L_2$ is given by $P_R(x, y)=R(x, y) \wedge \neg R(y, x)$ for all $x, y \in X$.

\begin{lema}
If $R: X^2 \rightarrow L_2$ is a crisp relation on $X$ then $\neg P_R(x, y)=R(x, y) \rightarrow R(y, x)$ for all $x, y \in X$.
\end{lema}

\begin{lema}
Let $R, Q$ be two crisp relations on $X$.

(i) $Q$ is a compatible extension of $R$ iff $R \subseteq Q$ and $Q(y, x) \leq R(x, y) \rightarrow R(y, x)$ for all $x, y \in X$;

(ii) $R$ is consistent iff $T(R)(y, x) \leq R(x, y) \rightarrow R(y, x)$ for all $x, y \in X$.
\end{lema}

\begin{proof}
Assume $R \subseteq Q$. By Lemma 4.1 we have for all $x, y \in X$:

$P_R(x, y) \leq P_Q(x, y)  \Leftrightarrow P_R(x, y) \leq Q(x, y) \wedge \neg Q(y, x) \Leftrightarrow$

$P_R(x, y) \leq \neg Q(y, x) \Leftrightarrow Q(y, x) \leq \neg P_R(x, y)=R(x, y) \rightarrow R(y, x)$

From these equivalences (i) follows immediately and (ii) is a particular case of (i).
\end{proof}

If $\ast$ is a left--continuous t--norm snd $R$ is a fuzzy relation on $X$, then the {\emph{asymmetric part}} of $R$ is a fuzzy relation $P_R$ on $X$
defined by $P_R(x, y)=R(x, y) \ast \neg R(y, x)$ for any $x, y \in X$.

Let $\wedge$ be the G\"{o}del t-norm. In this case the asymmetric part $P_R : X^2 \rightarrow [0, 1]$ has the form $P_R(x, y)=R(x, y) \wedge \neg R(y, x)$ for all $x, y \in X$. Following \cite{georgescu2}, a fuzzy relation $Q$ on $X$ is a {\emph{compatible extension}} of $R$ if $R \subseteq Q$ and $P_R \subseteq P_Q$; $R$ is {\emph{consistent}} if $T(R)$ is a compatible extension of $R$.

\begin{lema}
Assume that $\ast$ is the G\"{o}del t-norm.
Let $R, Q$ be two fuzzy relations on $X$.

(i) $Q$ is a compatible extension of $R$ if $R \subseteq Q$ and $Q(y, x) \leq \neg P_R(x, y)$ for all $x, y \in X$;

(ii) $R$ is consistent iff $T(R)(y, x) \leq \neg P_R(x, y)$ for all $x, y \in X$.
\end{lema}

\begin{proof}
(i) If $R \subseteq Q$ the following assertions are equivalent:

(1) $P_R \subseteq P_Q$;

(2) $P_R(x, y) \leq Q(x, y) \wedge \neg Q(y, x)$ for all $x, y \in X$;

(3) $P_R(x, y) \leq \neg Q(y, x)$ for all $x, y \in X$;

(4) $P_R(x, y) \wedge Q(y, x)=0$ for all $x, y \in X$;

(5) $Q(y, x) \leq \neg P_R(x, y)$ for all $x, y \in X$.

According to these equivalences (i) follows easily; (ii) is a particular case of (i).
\end{proof}

In the definitions of the notions of compatible extension and consistent consistent fuzzy relation the negation appears. According to Remark 3.4, in the case of the G\"{o}del t-norm
these notions do not describe adequately situations of fuzzy uncertainty. It is necessary to find some notion to elude negation, using possibly only the notion of implication. By Lemma
4.1, for the crisp relations, the asymmetric part is expressed with respect to the Boolean implication, thus the definitions existing in the crisp case satisfy this  desideratum.

\begin{remarca}
According to Lemma 4.3 (ii), the notion of consistency of a fuzzy relation from \cite{georgescu2} coincides with the notion of consistency defined in \cite{chaudhari}.
\end{remarca}

Since for the G\"{o}del t-norm Lemma 4.1 is not valid, we cannot express this definition of the compatible extension and the consistency only
in terms of the residuum.

Following the suggestion given by Lemma 4.2 we shall propose a different alternative definition, in which the residuum appears.

\begin{definitie}
Let $\ast$ be a left--continuous t--norm and the fuzzy relations $R, Q$ on $X$. $Q$ is called a $\ast$-compatible extension of $R$ if $R \subseteq Q$ and $Q(y, x) \leq R(x, y) \rightarrow R(y, x)$ for any $x, y \in X$. Relation $R$ is $\ast$--consistent if its $\ast$--transitive closure $T(R)$ of $R$ is a $\ast$--compatible extension of $R$, i.e. $T(R)(y, x) \leq R(x, y) \rightarrow R(y, x)$ for all $x, y \in X$.
\end{definitie}

\begin{remarca}
We assume that $\ast$ is the G\"{o}del t-norm $\wedge$. Let $Q$ be a $\wedge$--compatible extension of $R$. For any $x, y \in X$ we have $Q(y, x) \wedge P_R(x, y)=Q(y, x) \wedge R(x, y) \wedge \neg Q(y, x)=0$ therefore $Q(y, x) \leq \neg P_R(x, y)$. According to Lemma 4.3, $Q$ is a compatible extension of $R$.
\end{remarca}

\begin{exemplu}
Assume that $\ast$ is the G\"{o}del t-norm $\wedge$, $X=\{x, y\}$ and $R, Q$ the fuzzy relations on $X$ defined by

$R(x, x)=R(y, y)=1$; $R(x, y)=\frac{1}{2}$; $R(y, x)=\frac{1}{3}$;

$Q(x, x)=Q(y, y)=1$; $Q(x, y)=Q(y, x)=\frac{2}{3}$.

It is easy to see that $Q$ is a compatible extension of $R$ but

$R(x, y) \rightarrow R(y, x)=\frac{1}{2} \rightarrow \frac{1}{3}=\frac{1}{3}\leq \frac{2}{3}=Q(y, x)$.

Then $Q$ is not a $\wedge$--compatible extension of $R$.

In conclusion the $\wedge$--compatibility implies the compatibility but the converse is not true.
\end{exemplu}

\begin{remarca}
Assume that $\ast$ is the Lukasiewicz t--norm $\ast_L$. Then, for any $x, y \in X$, we have the equivalences:

$Q(x, y) \leq R(x, y) \rightarrow R(y, x)$ iff $Q(y, x) \leq 1- R(x, y) + R(y, x)$

thus $Q$ is a $\ast_L$--compatible extension of $R$ iff $R(x, y) \leq Q(x, y)$

and $R(x, y) +R(y, x) \leq 1-Q(y, x)$ for any $x, y \in X$.
\end{remarca}

\begin{propozitie}
If $\ast$ is a left--continuous t--norm and $R$ a fuzzy relation on $X$, then the following assertions are equivalent:

(i)$R$ is $\ast$--consistent;

(ii) For any integer $n \geq 1$ and for any $x, y, t_1, \ldots, t_n \in X$ we have

$R(y, t_1) \ast R(t_1, t_2) \ast \ldots \ast R(t_n, x)\leq R(x, y) \rightarrow R(y, x)$.

\end{propozitie}

\begin{proof}
In accordance with Lemma 3.6, $R$ is $\ast$--consistent iff for any $x, y \in X$ we have

(a) $R(y, x) \vee \displaystyle \bigvee_{n=1}^\infty \bigvee_{t_1, \ldots, t_n \in X} (R(y, t_1) \ast \ldots \ast R(t_n, x)) \leq R(x, y) \rightarrow R(y, x)$.

By the definition of the supremum, (a) is equivalent with the conjunction of the following two conditions:

(b) $R(y, x) \leq R(x, y) \rightarrow R(y, x)$;

(c) For any $n \geq 1$, $\displaystyle \bigvee_{t_1, \ldots, t_n \in X} (R(y, t_1) \ast \ldots \ast R(t_n, x)) \leq R(x, y) \rightarrow R(y, x)$.

The inequality (b) is always true (by Lemma 3.1, (4)) and (ii), (c) are equivalent assertions.
\end{proof}

\begin{remarca}
By Lemma 3.1 (1) and Proposition 4.9, $R$ is $\ast$--consistent iff for any integer $n \geq 1$ and $t_1, \ldots, t_n \in X$, $R(y, t_1) \ast \ldots \ast
R(t_n, x) \ast R(x, y) \leq R(y, x)$.
\end{remarca}

\begin{lema}
Assume $\ast$ is the G\"{o}del t-norm $\wedge$. If $Q$ is a $\wedge$--compatible extension of $R$ and $S$ is a $\wedge$--compatible extension of $Q$ then $S$ is a $\wedge$--compatible extension of $R$.
\end{lema}

\begin{proof}
For any $x, y \in X$ we have $R(x, y) \leq S(x, y)$ and $S(y, x) \wedge R(x, y) \leq S(y, x) \wedge Q(x, y) \leq Q(y, x)$ hence $S(y, x) \wedge R(x, y) \leq R(y, x)$.
\end{proof}

A family $\mathcal{R}$ of fuzzy relations on $X$ is said to be {\emph{closed upward}} if $\bigcup \{R|R \in \mathcal{C} \} \in \mathcal{R}$  for any chain $\mathcal{C} \subseteq \mathcal{R}$ .

\begin{lema}
The class of $\ast$--consistent fuzzy relations on $X$ is closed upward.
\end{lema}

\begin{proof}
Let $\{R_i\}_{i \in I}$ be a chain of $\ast$--consistent fuzzy relations on $X$ and $R=\displaystyle \bigcup_{i \in I} R_i$. We shall prove that $R$ is $\ast$--consistent.

Let $x, y \in X$. By Proposition 4.9 we must prove that for any integer $n \geq 1$ and $t_1, \ldots, t_n \in X$ the following inequality holds:

(a) $R(y, t_1) \ast \ldots \ast R(t_n, x) \leq R(x, y) \rightarrow R(y, x)$.

By the definition of $R$ and Lemma 3.2 (3):

(b) $R(y, t_1) \ast \ldots R(t_n, x)=\displaystyle [\bigvee_{i_1 \in I} R_{i_1}(y, t_1)] \ast \ldots \ast [\bigvee_{i_n \in I} R_{i_n}(t_n, x)]$

$=\displaystyle \bigvee_{i_1, \ldots, i_n \in I} R_{i_1}(y, t_1) \ast \ldots \ast R_{i_n}(t_n, x)$.

Let $i_1, \ldots, i_n \in I$. Since $\{R_i\}_{i \in I}$ is a chain there exists $j \in \{i_1, \ldots, i_n\}$ such that $R_{i_1}, \ldots, R_{i_n} \subseteq R_j$ hence

$R_{i_1}(y, t_1) \ast \ldots \ast R_{i_n}(t_n, x) \leq R_j(y, t_1) \ast \ldots \ast R_j(t_n, x)$.

Let $k \in I$ and $u \in \{j, k\}$ such that $R_j \subseteq R_u$, $R_k \subseteq R_u$. Thus

$R_j(y, t_1) \ast \ldots \ast R_j(t_n, x) \ast R_k(x, y) \leq R_n(y, t_1) \ast \ldots \ast R_n(t_n, x) \ast R_n(x, y) \leq$

$\leq T(R_n)(y, x) \ast R_n(x, y) \leq R_n(y, x) \leq R(y, x)$.

because $R$ is $\ast$--consistent. It follows that

$R_j(y, t_1) \ast \ldots R_j(t_n, x) \ast \leq R_k(x, y) \rightarrow R(y, x)$

therefore

$R_{i_1}(y, t_1) \ast \ldots \ast R_{i_n}(t_n, x) \leq R_k(x, y) \rightarrow R(y, x)$.

The last inequality holds for any $k \in I$, hence by Lemma 3.2 (2)

$R_{i_1}(y, t_1) \ast \ldots \ast R_{i_n}(t_n, x) \leq \displaystyle \bigwedge_{k \in I} (R_k(x, y) \rightarrow R(y, x))$

$\leq \displaystyle (\bigvee_{k \in I} R_k(x, y)) \rightarrow R(y, x)=R(x, y) \rightarrow R(y, x)$.

This inequality holds for all $i_1, \ldots, i_n \in I$ therefore

$\displaystyle \bigvee_{i_1, \ldots, i_n \in I} R_{i_1}(y, t_1) \ast \ldots \ast R_{i_n}(t_n, x) \leq R(x, y) \rightarrow R(y, x)$.

Thus, using (b), we obtain (a).
\end{proof}

\begin{lema}
\cite{georgescu2}
The class of $\ast$--transitive fuzzy relations on $X$ is closed upward.
\end{lema}

\begin{lema}
Let $R$ be a fuzzy relation on $X$. The class \hspace{0.2cm} $\mathcal{E}(R)$ of all $\ast$--compatible extensions of $R$ is closed upward.
\end{lema}

\begin{proof}
Let $(Q_i)_{i \in I}$ be a chain of $\ast$--compatible extensions of $R$. Let us denote $Q=\displaystyle \bigcup_{i \in I} Q_i$, hence $R \subseteq Q$. For any $x, y \in X$ we have

$Q(y, x) \ast R(x, y)=[\displaystyle \bigvee_{i \in I} Q_i(y, x)] \ast R(x, y)=\displaystyle \bigvee_{i \in I} Q_i(y, x) \ast R(x, y) \leq R(y, x)$

because any $Q_i$ is a $\ast$--compatible extension of $R$. By Lemma 3.1 (1) we get $Q(y, x) \leq R(x, y) \rightarrow R(y, x)$. Then
$Q$ is a $\ast$--compatible extension of $R$.
\end{proof}

\begin{lema}
The intersection of two closed upward classes of fuzzy relations is closed upward.
\end{lema}

If $R$ be a fuzzy relation on $X$ and $x, y \in X$ then we define the extension $R[x, y]$ of $R$ by

$R[x, y](a, b)=\left\{\begin{array}{rcl}
R(a, b) & \mbox{if}& (a, b) \not=(x, y) \\
1 & \mbox{if}& (a, b)=(x, y)
\end{array}\right.$

A class $\mathcal{R}$ of fuzzy relations on $X$ is {\emph{arc--receptive}} if $T(R[x, y])\in \mathcal{R}$ for any $\ast$--transitive relation $R \in \mathcal{R}$ and for all $x \not= y$ such that $R(x, y)=0$.

\begin{lema}
The intersection of two arc--receptive classes is arc--receptive.
\end{lema}

\begin{propozitie}
If $R$ is a $\ast$--transitive fuzzy relation on $X$, $x \not= y$ and $R(y, x)=0$, then $T(R[x, y])$ is a compatible extension of $R$.
\end{propozitie}

\begin{proof}
Let $a, b \in X$. To prove that $T(R[x, y])(b, a) \leq R(a, b) \rightarrow R(b, a)$ is equivalent to verifying the following two conditions:

(a) $R[x, y](b, a) \leq R(a, b) \rightarrow R(b, a)$;

(b) For any integer $n \geq 1$ and $t_1, \ldots, t_n \in X$,

$R[x, y](b, t_1) \ast \ldots \ast R[x, y](t_n, a) \leq R(a, b) \rightarrow R(b, a)$.

If $(b, a) \not=(x, y)$ then $R[x, y](b, a)=R(b, a)\leq R(a, b) \rightarrow R(b, a)$, according to Lemma 3.1 (4). If $(b, a)=(x, y)$ then $R(a, b)=R(y, x)=0$ and $R(a, b) \rightarrow R(b, a)=0 \rightarrow R(b, a)=1$, hence (a) is also verified.

Now we shall prove (b). Let $n \geq 1$ and $t_1, \ldots, t_n \in X$. Assume $(x, y)$ is distinct from all $(b, t_1), \ldots, (t_n, a)$. Then

$R[x, y](b, t_1) \ast \ldots \ast R[x, y](t_n, a) =R(b, t_1) \ast \ldots \ast R(t_n, a) \leq$

$\leq R(b, a) \leq R(a, b) \rightarrow R(b, a)$

because $R$ is $\ast$--transitive. We consider now the case when among the pairs $(b, t_1), \ldots, (t_n, a)$ there exists one equal to $(x, y)$. Denote
$t_0=b, t_{n+1}=a$. Let $t_k$ be the first occurrence of $x$ in the sequence $t_0, \ldots, t_{n+1}$ and $t_1$ be the last occurrence of $y$ in this sequence. Then
$(t_i, t_{i+1}) \not= (x, y)$ for $i=0, 1, \ldots, k-1$ and $i=1, \ldots, n$. This yields

$R[x, y](b, t_n) \ast \ldots \ast R[x, y](t_n, a) \ast R(a, b) \leq$

$R(b, t_1) \ast \ldots \ast R(t_{k-1}, x) \ast R(y, t_{l+1}) \ast \ldots \ast R(t_n, a) \ast R(a, b)=$

$=R(y, t_{l+1}) \ast \ldots \ast R(t_n, a) \ast R(a, b) \ast R(b, t_1) \ast \ldots \ast R(t_{k-1}, x) \leq R(y, x)=0 \leq R(b, a)$.

By Lemma 3.1 (1) we obtain

$R[x, y](b, t_1) \ast \ldots \ast R[x, y](t_n, a) \leq R(a, b) \rightarrow R(b, a)$

Then (b) was proved.
\end{proof}

\begin{teorema}
Let $\mathcal{R}$ be a class of fuzzy relations on $X$ satisfying the following conditions:

(i) $\mathcal{R}$ is closed upward and arc--receptive;

(ii) If $R, Q, S \in \mathcal{R}$, $Q$ is a $\ast$--compatible extension of $R$ and $S$ is a $\ast$--compatible extension of $Q$ then
$S$ is a $\ast$--compatible extension of $R$.

Then for any $\ast$--transitive fuzzy relation $R \in \mathcal{R}$ there exists a total and $\ast$--transitive fuzzy relation $R^\ast \in \mathcal{R}$ such that
$R^\ast$ is a $\ast$--compatible extension of $R$.
\end{teorema}

\begin{proof}
Let $\mathcal{C}$ be the class of $\ast$--transitive and $\ast$--compatible extensions of $R$ in $\mathcal{R}$. $\mathcal{C}$ is non--empty because $R \in \mathcal{C}$.
If $\mathcal{T}$ is the class of $\ast$--transitive fuzzy relations on $X$ then $\mathcal{C}=\mathcal{R} \cap \mathcal{E}(R) \cap \mathcal{T}$ (recall that $\mathcal{E}(R)$
is the class of all $\ast$--compatible extensions of $R$). According to Lemmas 4.13 and 4.14, $\mathcal{T}$ and $\mathcal{E}(R)$ are closed upward, hence, by Lemma 4.15,
$\mathcal{C}$ is closed upward. By Zorn's Axiom one gets a maximal member $R^\ast$ of $\mathcal{C}$, i.e. $R^\ast$ is a maximal $\ast$--transitive and $\ast$--compatible
extension of $R$ in $\mathcal{R}$.

It remains to prove that $R^\ast$ is total, i.e. for all distinct $x, y \in X$, $R^\ast(x, y)>0$ or $R^\ast(y, x)>0$. By absurdum, assume $R^\ast(x, y)=R^\ast(y, x)=0$ for some
distinct $x, y \in X$. Since $\mathcal{R}$ is arc--receptive, $x \not=y$ and $R^\ast(y, x)=0$ implies $T(R^\ast[x, y]) \in \mathcal{R}$.

$R^\ast$ fulfills the hypotheses of Proposition 4.17 ($R^\ast$ is $\ast$--transitive, $x \not=y$ and $R^\ast(y, x)=0$) therefore $T(R^\ast[x, y])$ is a $\ast$--compatible
extension of $R^\ast$. Since $R^\ast$ is a $\ast$--compatible extension of $R$ and $T(R^\ast[x, y])$ is a $\ast$--compatible extension of $R^\ast$ it follows that
$T(R^\ast[x, y])$ is a $\ast$--compatible extension of $R$ (by (ii)). Thus $T(R^\ast[x, y]) \in \mathcal{C}$. We remark that $1=R^\ast[x, y](x, y) \leq T(R^\ast[x, y])$
hence $T(R^\ast[x, y])(x, y)=1$. But $R^\ast(x, y)=0$ therefore $R^\ast$ is strictly included in  $T(R^\ast[x, y])$ contradicting the maximality of $R^\ast$. Thus $R^\ast$ is total
and the proof is finished.
\end{proof}

\begin{remarca}
The condition that $\mathcal{R}$ is closed upward ensures the existence of maximal members $R^\ast$ of $\mathcal{C}$ (by applying Zorn's Lemma), and arc-receptivity makes
$\mathcal{R}$  be closed to the operation $T(R^\ast[x, y])$, which leads to the fact that $R^\ast$ is total.
\end{remarca}

\begin{teorema}
Assume $\ast$ is the G\"{o}del t-norm $\wedge$ and the class $\mathcal{R}$ is closed upward and arc--receptive. Then any $\wedge$--transitive
fuzzy relation $R \in \mathcal{R}$ has a $\wedge$--compatible
extension $R^\ast$, total and $\wedge$--transitive.
\end{teorema}

\begin{proof}
By Lemma 4.12, for the case of the G\"{o}del t-norm condition (ii) holds.
\end{proof}

\begin{remarca}
Let us consider the case of the G\"{o}del t-norm $\wedge$. Since any $\wedge$--compatible extension is also compatible (Remark 5.4), from  Theorem 4.20 we get the main
result in \cite{georgescu2}(Theorem 4.7).
\end{remarca}

\begin{remarca}
Theorems 4.18 and 4.20 are general mathematical results which ensure the existence of the $\ast$--compatible extension, $\ast$--transitivity and totality of $\ast$--transitive
fuzzy relations. They generalize the main part of Duggan's results \cite{duggan1}, but also offer fuzzy versions of important extension theorems from the crisp case (see the next section).
\end{remarca}

\section{Particular cases}

This section deals with some particular cases of Theorem 4.20. This way we will obtain some fuzzy versions of classical extension theorems (Szpilrajn, Hansson, Suzumura).

In this section we will work with the G\"{o}del t-norm $\wedge$.

First we shall remind the definition of some classes of fuzzy relations. A fuzzy relation $R$ on $X$ is said to be:

$\bullet$ {\emph{fuzzy preorder}} if $R$ is reflexive and $\wedge$--transitive;

$\bullet$ {\emph{fuzzy weak order}} if $R$ is a total fuzzy preorder;

$\bullet$ {\emph{fuzzy strict partial order}} if $R$ is irreflexive and $\wedge$--transitive.

We shall denote

$\bullet$ $\mathcal{R}_1$= the class of fuzzy strict partial orders on $X$;

$\bullet$ $\mathcal{R}_2$= the class of fuzzy preorders on $X$;

$\bullet$ $\mathcal{R}_3$= the class of $\wedge$--transitive fuzzy relations on $X$.

\begin{lema}
\cite{georgescu2} The classes $\mathcal{R}_1$, $\mathcal{R}_2$, $\mathcal{R}_3$ are closed upward and arc--receptive.
\end{lema}

This lemma shows that Theorem 4.20 can be applied for classes $\mathcal{R}_1$, $\mathcal{R}_2$, $\mathcal{R}_3$.

\begin{teorema}
Any fuzzy strict partial order $R$ on $X$ has a $\wedge$--compatible extension $R^\ast$ which is a total fuzzy strict partial order on $X$.
\end{teorema}

\begin{remarca}
Theorem 5.2 is a fuzzy generalization of Szpilrajn theorem. The property of $\wedge$--compatibility which is verified by the extension $R^\ast$ of $R$ makes this theorem
be distinct from other fuzzy versions of Szpilrajn theorem (see \cite{bodenhofer}, \cite{gottwald}, \cite{zadeh}).
\end{remarca}

The following result generalizes Hansson's theorem.

\begin{teorema}
Any fuzzy preorder $R$ on $X$ has a $\wedge$--compatible extension $R^\ast$ which is also a fuzzy weak order.
\end{teorema}

\begin{teorema}
For a fuzzy relation $R$ on $X$ the following are equivalent:

(i) $R$ has a total and $\wedge$-transitive $\wedge$--compatible extension $Q$;

(ii) $R$ has a $\wedge$--transitive $\wedge$--compatible extension $Q$;

(iii) $R$ is $\wedge$--consistent.
\end{teorema}

\begin{proof}
$(i) \Rightarrow (ii)$ Obvious.

$(ii) \Rightarrow (iii)$ Assume that $R$ has a $\wedge$--transitive $\wedge$--compatible extension $Q$. Let $x, y \in X$. Since $Q$ is $\wedge$--transitive, we have
$T(R) \subseteq Q$. Thus, because $Q$ is a $\wedge$--compatible extension of $R$, the following inequalities hold:

$T(R)(y, x) \leq Q(y, x) \leq R(x, y) \rightarrow R(y, x)$.

It follows that $R$ is $\wedge$--consistent.
\end{proof}

The above result is a fuzzy version of Suzumura's theorem.

\begin{remarca}
By Remark 4.6, the $\wedge$--compatibility of fuzzy relations implies the compatibility defined in \cite{georgescu2}. Then, from Theorems 5.2, 5.4 and 5.5 obtained above, we obtain Theorems 5.8, 5.2 and 5.4 of \cite{georgescu2}. We
can also prove the corresponding result of Theorem 5.6, \cite{georgescu2}.
\end{remarca}

\section{Consistent closures of a fuzzy relation}

The consistent closure of a crisp relation was introduced in \cite{bossert1} in order to characterize the consistent rationalizability of crisp choice functions.
In this section with any fuzzy relation $R$ we shall assign two fuzzy relations $R^\Delta$, $R^\nabla$ which generalize the consistent closure. For the case of the
G\"{o}del t-norm we have $R^\Delta=R^\nabla$ and this construction is the smallest $\wedge$--consistent fuzzy relation including $R$.

Let $R$ be a crisp relation on $X$. Following \cite{bossert1}, the {\emph{consistent closure}} of $R$ is defined by

(6.1) $R^\ast=R \cup \{(x, y)|(x, y) \in T(R)$ and $(y, x) \in R \}$.

$R^\ast$ can also be written

(6.2) $R^\ast=T(R) \cap \{(x, y)|(x, y) \in R$ or $(y, x) \in R \}$.

\begin{remarca}
The consistent closure $R^\Delta$ has been introduced in \cite{chaudhari}, Definition 2.2 and further studied in \cite{chaudhari1}, \cite{chaudhari2}.
\end{remarca}

We want to define a similar notion for fuzzy relations.

Let $\ast$ be a left--continuous t--norm. In order to define the consistent closure of a fuzzy relation $R$ on $X$ we start from (6.1) and (6.2). Each of these
expressions suggest another form of the consistent closure of $R$.

If $R$ is a fuzzy relation on $X$ then we define the fuzzy relations $R^\Delta$ and $R^\nabla$ on $X$ by

(6.3) $R^\Delta (x, y)=R(x, y) \vee [T(R)(x, y) \ast R(y, x)]$

(6.4) $R^\nabla (x, y)=T(R)(x, y) \ast [R(x, y) \vee R(y, x)]$

for any $x, y \in X$. It is obvious that (6.3) (resp. (6.4)) generalizes (6.1) (resp. (6.2)).

\begin{lema}
Let $R$ be a fuzzy relation on $X$. Then

(i) $R^\nabla \subseteq R^\Delta$;

(ii) If $\ast$ is the G\"{o}del t-norm $\wedge$ then $R^\nabla = R^\Delta$.
\end{lema}

\begin{proof}
(i) By Lemma 3.2 (1)

$R^\nabla (x, y)=(T(R)(x, y) \ast R(x, y)) \vee (T(R)(x, y) \ast R(y, x))$

$\leq R(x, y) \vee T(R)(x, y) \ast R(y, x))=R^\Delta (x, y)$.

(ii) In case of $\wedge$ we have:

$R^\nabla (x, y)=(T(R)(x, y) \wedge R(x, y)) \vee (T(R)(x, y) \wedge R(y, x))$

$=R(x, y) \vee (T(R)(x, y) \wedge R(y, x))$

because $R(x, y) \leq T(R)(x, y)$.
\end{proof}

\begin{remarca}
Consider the set $X=\{x, y\}$ and the fuzzy relation $R$ on $X$ given by the matrix $R=\left( \begin{array}{cc}
1 & \frac{1}{3}  \\
\frac{1}{2} & 1  \end{array} \right)$.

Assume that $\ast$ is the product t--norm. We observe that $R$ is $\ast$--transitive then $R=T(R)$. An easy computation shows that
$R^\Delta (x, y)=\left( \begin{array}{cc}
1 & \frac{1}{3}  \\
\frac{1}{2} & 1  \end{array} \right)$;
$R^\nabla (x, y)=\left( \begin{array}{cc}
1 & \frac{1}{6}  \\
\frac{1}{4} & 1  \end{array} \right)$.

Then $R^\Delta \not=R^\nabla$ and $R$ is not included in $R^\nabla$.
\end{remarca}

\begin{propozitie}
Let $\ast$ be a left--continuous t--norm and $R$ a fuzzy relation on $X$.

(i) $R^\Delta \subseteq T(R)$;

(ii) $R^\nabla$ is $\ast$--consistent;

(iii) If $Q$ is a $\ast$--consistent fuzzy relation on $X$ such that $R \subseteq Q$ then $R^\nabla \subseteq Q$.
\end{propozitie}

\begin{proof}
(i) By the definition of $R^\Delta$.

(ii) Let $x, y, t_1, \ldots, t_n \in X$. Then

$R^\nabla(y, t_1) \ast \ldots \ast R^\nabla(t_n, x) \ast R^\nabla(x, y)$

$=T(R)(y, t_1) \ast [R(y, t_1) \vee R(t_1, y)] \ast \ldots \ast T(R)(t_n, x) \ast [R(t_n, x) \vee R(x, t_n)] \ast T(R)(x, y) \ast [R(x, y) \vee R(y, x)] \leq$

$\leq T(R)(y, t_1) \ast \ldots \ast T(R)(t_n, x) \ast [R(y, x) \vee R(x, y)]\leq$

$\leq T(R)(y, x) \ast [R(y, x) \vee R(x, y)]=R^\nabla(y, x)$

because $T(R)$ is $\ast$--transitive. By Remark 4.10, $R^\nabla$ is $\ast$--consistent.

(iii) Assume $Q$ is $\ast$--consistent and $R \subseteq Q$. For any $x, y \in X$ we have

$T(R)(y, x) \ast R(x, y) \leq T(Q)(y, x) \ast Q(x, y) \leq Q(y, x)$

$T(R)(y, x) \ast R(y, x) \leq R(y, x) \leq Q(y, x)$

hence

$R^\nabla(y, x)=T(R)(y, x) \ast R(y, x) \vee T(R)(y, x) \ast R(x, y) \leq Q(y, x)$.

Thus $R^\nabla \subseteq Q$.
\end{proof}

\begin{corolar}
Let $\ast$ be the G\"{o}del t-norm $\wedge$. If $R$ is a fuzzy relation on $X$ then $R^\Delta=R^\nabla$ is the smallest $\wedge$--consistent fuzzy relation on $X$ containing $R$.
\end{corolar}

According to Corollary 6.5, the fuzzy relation $R^\ast=R^\Delta=R^\nabla$ is said to be the $\wedge$--consistent closure of $R$. If $R^{-1}$ is the converse relation of
$R$ ($R^{-1}(x, y)=R(y, x)$ for all $x, y \in X$) then $R^\ast=T(R) \cap (R \cup R^{-1})$.

\begin{exemplu}
Consider the set $X=\{x, y, z\}$ and the fuzzy relation $R$ given by the matrix $R=\left( \begin{array}{ccc}
1 & \frac{1}{3} &  1 \\
\frac{1}{4} & 1  & \frac{1}{2} \\
\frac{1}{2} & 1 & 1 \end{array} \right)$.

By computation we obtain $T(R)=\left( \begin{array}{ccc}
1 & 1 &  1 \\
\frac{1}{2} & 1  & \frac{1}{2} \\
\frac{1}{2} & 1 & 1 \end{array} \right)$

therefore

$R^\ast=T(R) \cap [R \cup R^{-1}] = \left( \begin{array}{ccc}
1 & \frac{1}{3} &  1\\
\frac{1}{3} & 1  & \frac{1}{2}  \\
\frac{1}{2}  & 1 & 1 \end{array} \right)$.
\end{exemplu}

\section{Conclusions}

This paper contains the following contributions:

(a) A new notion of compatible extension of a fuzzy relation is defined and on its basis a new type of consistency of fuzzy relations is introduced. These two notions
are different from those of \cite{georgescu2}, against which they have the advantage of being defined using only the residuum $\rightarrow$ associated with a left--continuous
t--norm $\ast$ \footnote{We talk about Theorem 2.4, that part of Duggan's result which generalizes the theorems of Szpilrajn, Hansson and Suzumura.}

(b) These new notions allow to obtain a general extension theorem which generalizes a significant part \footnote{Avoiding the negation in formulating these definitions is necessary
to elude those cases when this operator takes only the extreme values $0$ and $1$ (for example, in case of the G\"{o}del t-norm)} of Duggan's extension theorem from the case
of crisp relations. From this general result are obtained as particular cases fuzzy versions of some important extension theorems for crisp relations  (Szpilrajn, Hansson and Suzumura).

(c) The notion of consistent closure of a crisp relation \cite{bossert1} is generalized to fuzzy relations by two constructions $R^\Delta$ and $R^\nabla$, intrinsicly related
to the type of consistency from this paper.

We present next a few open problems which come from the results of this paper.

(1) To prove a fuzzy version of Duggan's General Extension Theorem (see Theorem 2.5) allowing to obtain fuzzy forms of Dushnik--Miller and Donaldson--Weymark
theorems.

(2) To obtain a result generalizing Theorems 4.16 and 4.17, as well as their consequences, in the context offered by fuzzy relations indicators (see
\cite{georgescu3}, pp. 66-71, \cite{wang}).

(3) The consistency of crisp relations has been intensively used in social choice theory (see \cite{suzumura2} or \cite{bossert3}). It would be interesting to study whether the notion of fuzzy consistency
and the results of this paper can be connected to fuzzy social choice theory \cite{mordeson}.

\end{document}